\pdfoutput=1

\documentclass[11pt]{article}

\usepackage[]{acl}

\usepackage{times}
\usepackage{latexsym}

\usepackage[T1]{fontenc}

\usepackage[utf8]{inputenc}

\usepackage{microtype}

\usepackage{bm}
\usepackage{amsthm}
\usepackage{amsmath}
\usepackage{xcolor}
\usepackage{graphicx}
\usepackage{comment}
\usepackage{enumitem}

\usepackage{amsthm}
\usepackage{aliascnt}
\theoremstyle{definition}
\newtheorem{theorem}{Theorem}

\newaliascnt{prop}{theorem}
\newtheorem{prop}[prop]{Proposition}
\aliascntresetthe{prop}

\makeatletter
\newcommand{\bhline}[1]{\noalign{\hrule height #1}}  
 
\newcommand{\PMI}{\mathrm{PMI}}
\newcommand{\word}[1]{\textit{#1}}
\newcommand{\eg}{e.g.}
\newcommand{\ie}{i.e.}

\makeatother

\title{Revisiting Additive Compositionality: AND, OR, and NOT Operations with Word Embeddings}

\author{
  Masahiro Naito$^{1,3}$ \qquad Sho Yokoi $^{2,3}$ \qquad Geewook Kim$^4$ \qquad Hidetoshi Shimodaira $^{1,3}$ \\
  $^1$Kyoto University \qquad $^2$Tohoku University \qquad $^3$RIKEN \qquad $^4$ Naver\\
  \texttt{neiteng@sys.i.kyoto-u.ac.jp, yokoi@ecei.tohoku.ac.jp,}\\
  \texttt{kdrl7@naver.com, shimo@i.kyoto-u.ac.jp}\\
}

\begin{document}
\maketitle
\begin{abstract}
    It is well-known that typical word embedding methods have the property that the meaning can be composed by adding up the embeddings (additive compositionality). Several theories have been proposed to explain additive compositionality, but the following problems remain: (i) The assumptions of those theories do not hold for practical word embedding. (ii) Ordinary additive compositionality can be seen as an AND operation of word meanings, but it is not well understood how other operations, such as OR and NOT, can be computed by the embeddings. We address these issues with the idea of frequency-weighted centering at its core. This method bridges the gap between practical word embedding and the assumption of theory about additive compositionality as an answer to (i). This paper also gives a method for taking OR or NOT of the meaning by linear operation of word embedding as an answer to (ii). Moreover, we confirm experimentally that the accuracy of AND operation, i.e., the ordinary additive compositionality, can be improved by our post-processing method (3.5x improvement in top-100 accuracy) and that OR and NOT operations can be performed correctly. We also confirm that the proposed method is effective for BERT embeddings.
\end{abstract}

\section{Introduction} \label{sec:intro}

\begin{figure}
    \centering
    \includegraphics[width=75mm]{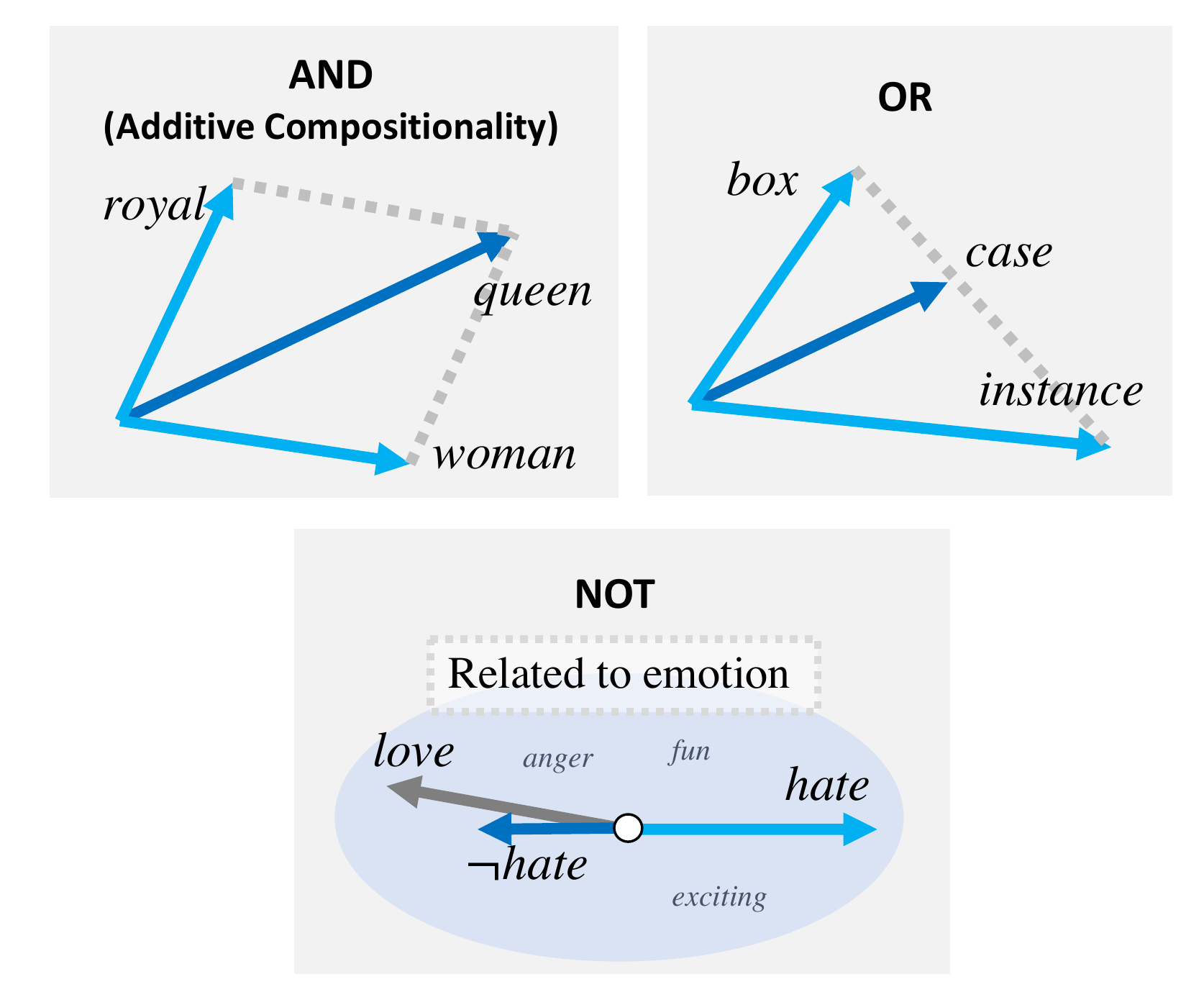}
    \caption{Illustration of AND, OR, and NOT operations with word embeddings.}
    \label{fig:and_or_not_illust}
\end{figure}

Word embedding \cite{sgns, glove, bert}, a fundamental technology in natural language processing, has the property that meaning can be composed by adding up the embeddings. This property is called \emph{additive compositionality}, \eg, $\bm{v}_{\word{king}} \approx \bm{v}_{\word{royal}} + \bm{v}_{\word{man}}$ \cite{sgns}. In this paper, we raise and resolve two questions about additive compositionality, as described below.

(i) Are the existing theories of additive compositionality realistic and unified? Several theories have been proposed to explain why additive compositionality holds \cite{arora-latent,gittens,allen}; however, assumptions that do not hold for SGNS and GloVe have been made in these papers. Besides, since these theories depend on specific methods such as Skip-Gram \cite{skipgram}, a unified understanding with other methods such as GloVe \cite{glove}, and BERT \cite{bert} remains a challenge in the field of natural language processing.

(ii)
As will be explained later, ordinary additive compositionality corresponds to AND in logical operations. This means that AND corresponds to the addition of embeddings; what embedding operations do OR and NOT correspond to? Examples of OR are $\word{case} \approx \word{box} \lor \word{instance}$ (polysemous word / homograph) and $\word{child} = \word{boy} \lor \word{girl}$ (hypernym-hyponym). An example of NOT is $\word{hate} \approx \lnot \word{love}$ (antonym).

\begin{figure}
    \centering
    \includegraphics[width=75mm]{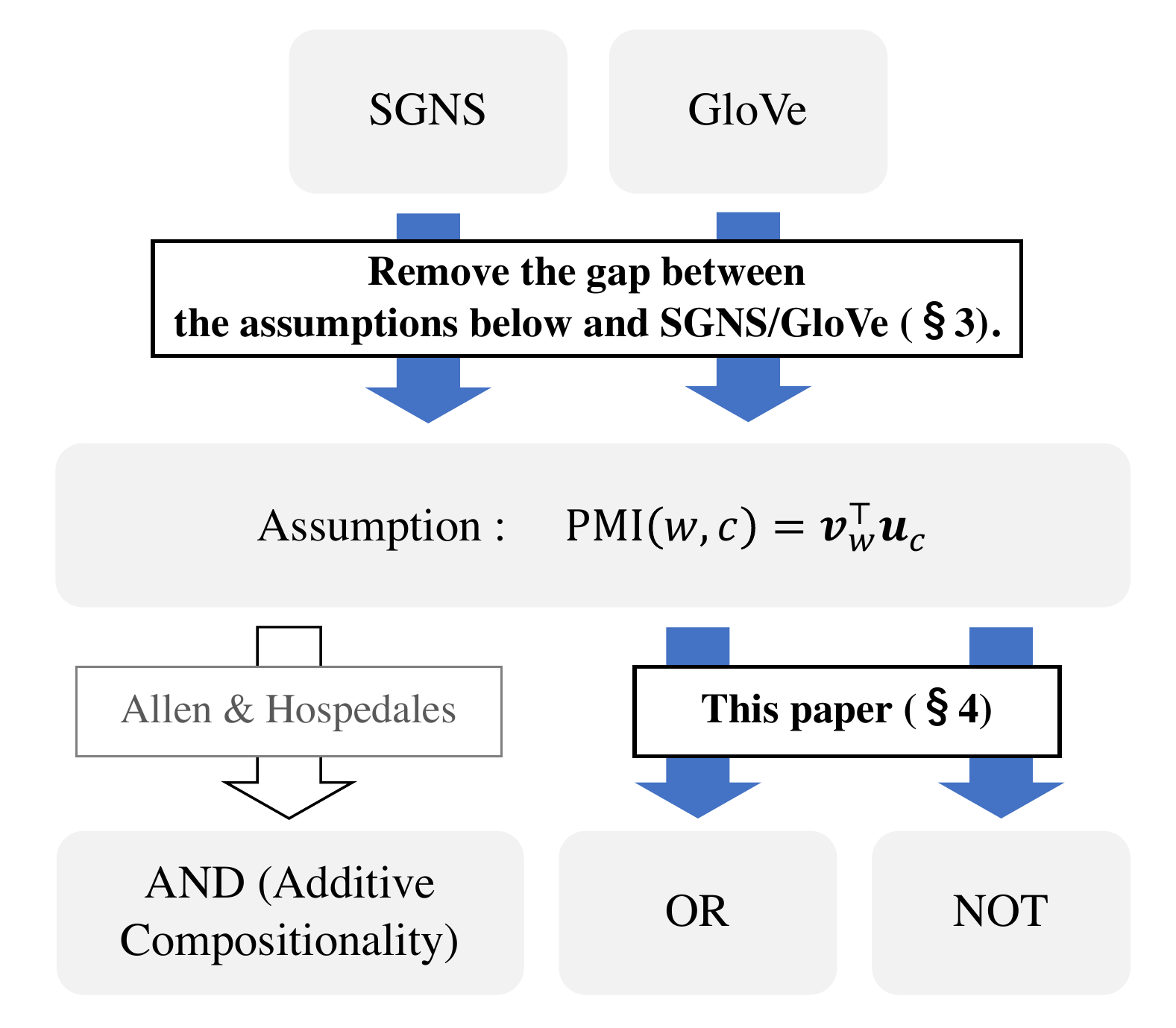}
    \caption{An overview of this paper and the previous research. }
    \label{fig:punch}
\end{figure}

We provide the following theoretical and experimental contributions to the above-mentioned questions.
\begin{enumerate}
    \item We show that the gap between SGNS/GloVe and the assumption of \cite{allen} almost disappears when word embeddings are centered using the frequency-weighted average of vocabulary words. In other words, by coupling our theory with \cite{allen} we explain the mechanism of additive compositionality in SGNS and GloVe. Besides, we propose a method to make additive compositionality more strongly hold.
    Our theory holds for both \textbf{SGNS} (Skip-Gram with Negative Sampling, one of the variations of Word2Vec by \citealt{sgns}) and \textbf{GloVe} \cite{glove}. This implies that centering allows SGNS and GloVe to be described in an almost unified form. Pointing out the similarities between the limited BERT architecture and SGNS, we suggest that frequency-weighted centering may be applicable to BERT embeddings as well. \label{enum:intro_pmi}
    \item Utilizing the results obtained in \ref{enum:intro_pmi}. as a starting point, we extend the theory of ordinary additive compositionality (AND) to compositionality of OR and NOT (see \autoref{fig:and_or_not_illust}).
    OR operation is a frequency-weighted average for a specified subset of vocabulary words. NOT operation is based on a novel \emph{conditional embedding} that is computed by frequency-weighted centering. 
    
    \item We experimentally confirm that our theory is correct (\autoref{sec:exp}). The experimental results show that frequency-weighted centering makes additive compositionality, which corresponds to AND operation, holds more accurately (3.5x improvement in top-100 accuracy). We also confirm that this method is effective for BERT embeddings. We also showed that the proposed formula can successfully compute OR and NOT embeddings.
    
\end{enumerate}

\section{Preliminaries: Word Embedding}

In this section, 
we briefly introduce some properties of popular word embedding methods. In the next section, we point out the gap between these properties and the assumption of the theory of additive compositionality \cite{allen}, and propose a method to resolve it.

Word embedding methods based on co-occurrence information between words, such as SGNS and GloVe, are used across a wide range of fields, such as information retrieval and recommendation systems \cite{info_ret,node2vec,prod2vec}.

SGNS and GloVe (and maybe BERT) encode the co-occurrence information of words.
\citet{levy} showed that optimally trained SGNS embedding satisfies
\begin{align}
    \log \frac{p(w, c)}{p(w)q(c)} - \log k = \bm{v}_w^\top \bm{u}_c, \label{eq:sgns_fac}
\end{align}
where $p$ is the word distribution of corpus, $q$ is the distribution of negative samples, $k$ is the number of negative samples per co-occurring word pair $(w,c)$, $\bm{v}_w$ is the embedding of a target word $w$, and $\bm{u}_c$ is the embedding of a context word $c$.
GloVe~\cite{glove} takes a direct approach to factorize the co-occurrence matrix, and the optimally learned embedding satisfies
\begin{align}
    \log p(w,c) = \bm{v}_w^\top \bm{u}_c + a_w + b_c - \log Z, \label{eq:glove_fac}
\end{align}
where $a_w, b_c$ are bias terms and $Z$ is a normalization constant.
In the following, we assume that SGNS and GloVe satisfy (\ref{eq:sgns_fac}) and (\ref{eq:glove_fac}), respectively.

\section{Structure Common to SGNS and GloVe} \label{sec:pmi}

In this section, we show that when SGNS and GloVe are centered using the frequency-weighted average, they share a common structure.

\citet{allen} explained additive compositionality with the assumption
\begin{align}
    \PMI(w,c) = \bm{v}_w^\top\bm{u}_c \label{eq:pmi}
    \text{,}
\end{align}
where $\PMI(w,c)$ is the pointwise mutual information (PMI) between $w$ and $c$
\begin{align}
    \PMI(w, c) := \log \frac{p(w, c)}{p(w)p(c)}
    \text{.}
\end{align}
However, neither SGNS nor GloVe satisfies the assumption (\ref{eq:pmi}), as will be explained in \autoref{subsec:err_term}. If we can adjust the word embeddings to satisfy assumption (\ref{eq:pmi}), then additive compositionality should hold more accurately. 

In this section, we show a simple post-processing method for this adjustment of word embeddings, which can be applied to both SGNS and GloVe.

\subsection{Error Terms in (\ref{eq:pmi})} \label{subsec:err_term}

Rearranging the formulas of the word embedding assumptions (\ref{eq:sgns_fac}) and (\ref{eq:glove_fac}), we have
\begin{align}
    \textbf{SGNS}\qquad &\PMI(w,c) \nonumber \\
    & = \bm{v}_w^\top \bm{u}_c + \log \frac{q(c)}{p(c)} +  \log k, \label{eq:sgns_model_pmi} \\
    \textbf{GloVe}\qquad &\PMI(w,c) \nonumber \\
    & = \bm{v}_w^\top \bm{u}_c + (a_w - \log p(w)) \nonumber \\
    & \mathrel{\hphantom{=}} + (b_c - \log p(c)) - \log Z. \label{eq:glove_model_pmi}
\end{align}
Clearly, they differ from (\ref{eq:pmi}).  \citet{allen} ignores the second and subsequent terms on the right-hand side of (\ref{eq:sgns_model_pmi}) and (\ref{eq:glove_model_pmi}); these ignored terms are considered as error terms in the assumption~(\ref{eq:pmi}). Experiments in this paper show that these error terms are \textit{not} negligible (\autoref{subsec:exp_pmi}).

\subsection{Frequency-weighted Centering} \label{subsec:freq_centering}

First, we show that (\ref{eq:pmi}) can be derived by centering the SGNS/GloVe embedding in a form that includes some error terms. For the word embeddings $\bm{v}_w$ and $\bm{u}_c$, the frequency-weighted averages of word embeddings are
\begin{align} \label{eq:frequency-weighted-average}
\bm{\bar{v}} = \sum_w p(w) \bm{v}_w,\quad
\bm{\bar{u}} = \sum_c p(c) \bm{u}_c,
\end{align}
and the centered word embeddings are
\begin{align} \label{eq:frequency-weighted-centering}
\bm{\tilde{v}}_w = \bm{v}_w - \bm{\bar{v}},\quad
\bm{\tilde{u}}_c = \bm{u}_c - \bm{\bar{u}}.
\end{align}

\begin{theorem}
    \label{th:pmi_th}
    When the embedding of SGNS and GloVe satisfies (\ref{eq:sgns_fac}) and (\ref{eq:glove_fac}), respectively, the following equality holds:
    \begin{align}
        \PMI(w,c) = \bm{\tilde{v}}_w^\top \bm{\tilde{u}}_c + \bar{\epsilon} - \epsilon_w - \epsilon_c, \label{eq:pmi_th}
    \end{align}
    where the error terms are defined, with KL-divergence, as $\epsilon_w = D_{\mathrm{KL}} (p(\cdot) \| p(\cdot | w))$, $\epsilon_c = D_{\mathrm{KL}} (p(\cdot) \| p(\cdot | c))$, and $\bar{\epsilon} = \sum_w p(w) \epsilon_w$.
\end{theorem}

\begin{proof}
    See \autoref{sec:pmi_proof}. 
\end{proof}

The following proposition shows that the error terms are negligible when $|\PMI(w,c)|\ll 1$.

\begin{prop} \label{prop:pmi_order}
    Let $\Delta = \max_{w,c} |\PMI(w,c)|$. For sufficiently small $\Delta$,  $\epsilon_w = O(\Delta^2), \bar\epsilon = O(\Delta^2)$.
\end{prop}
\begin{proof}
    See \autoref{sec:pmi_order_proof}. 
\end{proof}

\subsection{Discussion}

\paragraph{Interpretation} Theorem~\ref{th:pmi_th} 
suggests that the centered SGNS and GloVe can be described in roughly the same form. In other words, we can say that (\ref{eq:pmi}) is a structure essentially common to SGNS and GloVe if properly centered.

\paragraph{Relation to experimental results} The experiments described below confirm that the error in assumption (\ref{eq:pmi}) is significantly reduced by the frequency-weighted centering (\autoref{subsec:exp_pmi}), which supports the theory in \autoref{subsec:freq_centering}. Furthermore, we have confirmed that the accuracy of additive compositionality is improved by frequency-weighted centering, as we expected. These improved word vectors are applicable to various downstream tasks.

\paragraph{Comparison with All-but-The-Top} \citet{abtt} suggested that uniform centering $\bm{v}_w \leftarrow \bm{v}_w - \sum_c \bm{v}_c / |V|$ is helpful as a post-processing method to get high-performance word embeddings. This method is described as adjusting the embeddings to satisfy isotropy, a property that the RAND-WALK model \cite{arora-latent} should satisfy. However, its argument is not complete because the theoretical basis for RAND-WALK's high performance on each downstream task is not clearly stated. On the other hand, since our method is designed with the goal of satisfying the assumption of the theory of additive compositionality \cite{allen}, there is a direct connection between our theory and the experimental results.

\section{Logical Operations with Word Embeddings} \label{sec:logical}

In this section, we point out that ordinary additive compositionality is an AND-like operation. We show that other logical operations, such as OR and NOT, can also be computed from embeddings. We adopt assumption (\ref{eq:pmi}) in this section as well as \citet{allen}; embeddings satisfying (\ref{eq:pmi}) can be obtained by the simple post-processing of SGNS and GloVe (see \autoref{sec:pmi}).

\subsection{AND Operation} \label{subsec:th_and}

\citet{allen} showed that when the PMI factorization structure (\ref{eq:pmi}) is strictly satisfied, a semantic AND composite such as $\textit{queen} = \textit{royal} \land \textit{woman}$ corresponds to vector additivity such as the following formula:
\begin{align}
    \bm{v}_{\textit{royal}} = \bm{v}_{\textit{royal}} + \bm{v}_{\textit{woman}} . \label{eq:and_ex}
\end{align}
In this section, we outline the proof of \citet{allen}. 

\subsubsection{Formulation with Co-occurrence Probability}
Let $w = w_1 \land w_2 \land \cdots \land w_s$. Let us assume, for example, that the probability of occurrence of \word{queen} meaning is the multiplication of the probabilities of \word{royal} and \word{woman} meaning. Generalizing this, we formulate AND-like compositionality as follows:
\begin{align}
    \forall c \in V, \quad p(w|c) &= p(w_1|c) \cdots p(w_s|c) \label{eq:and_modeling_c}, \\
    p(w) &= p(w_1) \cdots p(w_s) \label{eq:and_modeling}
\end{align}

\subsubsection{Computation on Embedding Space}

From the above formulation, additive compositionality is proved.
\begin{theorem}[\citealt{allen}] \label{th:and_th}
    When $w, w_1, \ldots, w_s$ satisfy (\ref{eq:pmi}), (\ref{eq:and_modeling_c}) and (\ref{eq:and_modeling}), 
    \begin{align}
        \bm{v}_w = \sum_{i = 1}^s \bm{v}_{w_i} . \label{eq:and_th_v}
    \end{align}
\end{theorem}
\begin{proof}
    Dividing (\ref{eq:and_modeling_c}) by (\ref{eq:and_modeling}) and taking the logarithm, we get $\PMI(w,c) = \PMI(w_1, c) + \cdots + \PMI(w_s, c)$. From (\ref{eq:pmi}), $\bm{v}_w^\top\bm{u}_c = \bm{v}_{w_1}^\top\bm{u}_c+\cdots+\bm{v}_{w_s}^\top\bm{u}_c$. Since $c \in V$ is arbitrary, (\ref{eq:and_th_v}) follows.
\end{proof}

\subsection{OR Operation} \label{subsec:th_or}

As mentioned in \autoref{sec:intro}, in addition to AND operation, OR operation can also be considered. In this section, we show that OR operation corresponds to the frequency-weighted average of the embeddings for a set of words. 

\subsubsection{Formulation with Co-occurrence Probability}

OR operation is denoted by operator $\lor$. Let $w$ be the OR word of $w_1, w_2, \ldots, w_s$, \ie~$w = w_1 \lor w_2 \lor \cdots \lor w_n$. For example, $\textit{case} \approx \textit{box} \lor \textit{instance}$. The probability of occurrence of $w$ in each context $c$ can be formulated as the sum of the probabilities of occurrence of $w_1, \ldots, w_s$:
\begin{align}
    \forall c \in V, \quad p(w|c) =  p(w_1|c) + \cdots + p(w_s|c) . \label{eq:or_modeling}
\end{align}
From (\ref{eq:or_modeling}), we get $p(w) = \sum_{i = 1}^s p(w_i)$.

\subsubsection{Computation on Embedding Space}

On the basis of the above simple formulation, we give a method to perform OR operation on the embeddings. 

\begin{theorem}[OR formula] \label{th:or_th}
    We assume that $w, w_1, w_2, \ldots, w_s$ satisfy (\ref{eq:pmi}) and (\ref{eq:or_modeling}). When $|\PMI(w,c)| \ll 1$, word embeddings satisfy
    \begin{align}
        \bm{v}_w \approx \sum_{i = 1}^s \frac{p(w_i)}{p(w)} \bm{v}_{w_i} . \label{eq:or_th_v}
    \end{align}
\end{theorem}
\begin{proof}
    See \autoref{sec:or_proof}. 
\end{proof}
OR formula (\ref{eq:or_th_v}) suggests that the embedding of \textit{case} approximates the sum of the embeddings of \textit{box} and \textit{instance}, weighted by their probability of occurrence in the corpus. Note that the OR formula is invariant to the translation of the origin, so it is valid to some extent for SGNS and GloVe without the frequency-weighted centering.

\subsubsection{Discussion}

\paragraph{Relation to experimental results} On the real data, $|\PMI(w,c)|\ll 1$ does not strictly hold, but we confirmed that the OR formula holds well in the experiment in \autoref{subsec:exp_or}. 

\paragraph{Comparison with previous work} \citet{arora-polysemy} obtained the same formula by assuming the random walk of the context vector (RAND-WALK model), but the proof in this paper does not require that assumption.

\subsection{Conditional Embedding and NOT Operation}

With assumption (\ref{eq:pmi}), we derive not only AND and OR operations but also NOT operation. In this section, we formulate the NOT operation using the concept of \emph{conditional embedding}, word embedding that expresses the local relationship between words in a small set of words $A\subset V$. We derive that the conditional embedding of the antonym is proportional to \emph{minus} of the conditional embedding of the original word.

\subsubsection{Formulation with Co-occurrence Probability}

In contrast to human senses, antonyms have the property of being dissimilar and similar at the same time \cite{cruse,willners}, \eg, \textit{hate} and \textit{love} have opposite meanings, but both of them are related to emotion. For this reason, antonyms tend to appear in similar contexts, and thus their word embeddings trained by the method based on the distributional hypothesis \cite{harris54,firth} exhibit a high similarity\footnote{The cosine similarity between \textit{hate} and \textit{love} is about 0.5 for vectors obtained by SGNS (without centering).}. Therefore, antonyms are related to synonyms, making it difficult to understand how they are embedded. In this section, we dispense with the mystery of antonyms by formulating them in a way that takes their similarity into account.

Let us take the following example: the opposite of \textit{mother} is \textit{father} in the ``parent" category, but \textit{daughter} in the ``parent-child relationship" category. In this way, when considering antonyms, one needs to specify a category corresponding to the similarity portion of the antonyms. In this paper, a category is represented by a set of words $A$. From the intuition that the antonym $\lnot w$  of word $w\in A$ corresponds to the complement of $w$ in $A$ when viewed in a small word set $A$, the co-occurrence probability of NOT word can be formulated by the following conditional probability:
\begin{multline}
    p(W = \lnot w \mid W \in A, c) \\ =p(W \in A \setminus \{w\} \mid W \in A, c) , \label{eq:not_modeling}
\end{multline}
where word $W$ denotes a random variable and $p(\cdot | \cdot)$ denotes the conditional probability.
Because the event $W \in A$ appears in the conditioning for the probability of (\ref{eq:not_modeling}), we need embeddings conditioned on $A$ instead of the whole vocabulary. In this paper, we refer to this embedding as \emph{conditional embedding on $A$}.

\subsubsection{Conditional Embedding}

From the analogy to (\ref{eq:pmi}), we consider the equality to be satisfied by the conditional embedding $\bm{v}_{w|A}$ of the word $w$ in set $A$ as follows:
\begin{multline}
    p(W = w \mid W \in A, c) \\
    = p(W = w \mid W \in A) \exp(\bm{v}_{w|A}^\top \bm{u}_{c}) \label{eq:cond_model}.
\end{multline}
From the following \autoref{th:cond_emb}, we can see that conditional embedding can be approximated by frequency-weighted centering on a subset $A$.
\begin{theorem} \label{th:cond_emb}
    When embeddings satisfy (\ref{eq:pmi}),
    \begin{align}
        \bm{v}_{w|A} \approx \bm{v}_{w}- \bm{v}_A, \label{eq:cond_emb_calc}
    \end{align}
    where $p(A) = p(W \in A) = \sum_{w \in A} p(w)$ and $\bm{v}_A = \sum_{w \in A} \frac{p(w)}{p(A)} \bm{v}_w$.
\end{theorem}
\begin{proof}
    See \autoref{sec:cond_emb_proof}. 
\end{proof}
\autoref{th:cond_emb} allows us to explain the common practice of centering, although typically unweighted, on a particular set of words, \eg, implicit centering in PCA visualization.

\subsubsection{Computation on Embedding Space}

On the basis of the above formulation, we derive a method for computing NOT with word embeddings.

\begin{theorem}[NOT formula] \label{th:not_th}
    Assuming that the words $w$ and $\lnot w | A$ satisfy (\ref{eq:cond_model}), we have
    \begin{align} \label{eq:not-formula-theorem}
        \bm{v}_{\lnot w | A} \approx -\frac{p(W = w \mid W \in A)}{1-p(W = w \mid W \in A)} \bm{v}_{w | A} .
    \end{align}
\end{theorem}
\begin{proof}
    See \autoref{sec:not_proof}. 
\end{proof}

 From this formula, we can see that the conditional embedding of the NOT word of $w$ is the vector in the negative direction of the conditional embedding of the original word $w$. 

\subsection{Extension to BERT}

BERT \cite{bert}, which has attracted attention in recent years, obtains word embeddings by predicting a word from its context, and can be regarded as an extension of SGNS for the following reasons. Consider a one-layer BERT model pre-trained by masked LM only. If the attention weight of a [MASK] token is a one-hot vector, BERT predicts [MASK] from one context word and can be regarded as a Skip-gram model \cite{sgns}. Consider the input sentence $c_1 c_2 \cdots c_\ell$ including [MASK].
Setting the attention weight $\bm{\alpha}$ to $\bm{\alpha} = \begin{bmatrix} 0,\ldots, 0,1, 0, \ldots, 0 \end{bmatrix}^\top$, one-hot vector with 1 at $i$-th element, BERT's probability model is
\begin{align}
    &p(\text{[MASK]} = w \mid \text{sentence $c_1 c_2 \cdots c_\ell$}) \nonumber \\
    &\propto \exp
    \left(\bm{u^{\text{BT}}}_w \cdot
        \left(
            \begin{bmatrix}
                \bm{v}_{c_1}^{\text{BT}} & \cdots & \bm{v}_{c_\ell}^{\text{BT}}
            \end{bmatrix}
            \begin{bmatrix}
                \alpha_1 \\
                \vdots \\
                \alpha_\ell
            \end{bmatrix}
        \right)
    \right) \nonumber \\
    &= \exp \left(\bm{u}_w^{\text{BT}} \cdot \bm{v}_{c_i}^{\text{BT}} \right), \label{eq:bert_alpha_onehot}
\end{align}
where $\bm{v}^{\text{BT}}$ is input representation and $\bm{u}^{\text{BT}}$ is output representation.
(\ref{eq:bert_alpha_onehot}) is similar to Skip-gram model $p(w|c) \propto \exp(\bm{v}_c^\top \bm{u}_w)$.
Thus, BERT can be regarded as a generalization of Skip-gram. Therefore, we can expect that our proposed method will apply to BERT to some extent. Some experiments in \autoref{sec:exp} also provide results in BERT.

\section{Experiments} \label{sec:exp}

In order to keep the description concise, the detailed experimental setup is described in \autoref{sec:detail_exp}. 

\newcommand{\origc}[1]{\colorbox[rgb]{0.92,0.92,0.92}{\textcolor[rgb]{0,0,0}{#1}}}
\newcommand{\unifc}[1]{\colorbox[rgb]{0.92,0.96,1}{\textcolor[rgb]{0,0.69,1}{#1}}}
\newcommand{\freqc}[1]{\colorbox[rgb]{1,0.9,0.9}{\textcolor[rgb]{1,0,0}{#1}}}
\newcommand{\abttc}[1]{\colorbox[rgb]{0.95,0.85,1}{\textcolor[rgb]{0.5,0,0.82}{#1}}}
\newcommand{\freq}{\textsf{\freqc{freq}}}
\newcommand{\orig}{\textsf{\origc{orig}}}
\newcommand{\unif}{\textsf{\unifc{unif}}}
\newcommand{\abtt}{\textsf{\abttc{ABTT}}}

\subsection{Centering and PMI Factorization}  \label{subsec:exp_pmi}
In this section, we experimentally confirm that (\ref{eq:pmi}) holds more accurately if we perform frequency-weighted centering (\autoref{subsec:freq_centering}). To show that the accuracy of PMI factorization formula $\PMI(w,c) = \bm{v}_w^\top \bm{u}_c$ is improved by centering the embeddings, we observed the distribution of the error $e_{wc} = \PMI(w, c) - \bm{v}_w^\top \bm{u}_c$ in several experimental settings.

\paragraph{Embeddings} We used 300-dimensional embeddings trained by SGNS and GloVe with text8 corpus\footnote{\url{http://mattmahoney.net/dc/textdata.html}}.
The results are shown for the following three sets of embeddings.
\begin{itemize}
    \item \orig : original embeddings.
    \item \unif : embeddings with uniform centering\footnote{This is a standard post-processing of word embedding \cite{abtt}.} ~\cite{abtt}, \ie, $\bm{v}_w \leftarrow \bm{v}_w - \sum_{w'} \bm{v}_{w'} / |V|$.
    
    \item \freq : embeddings with frequency-weighted centering.
\end{itemize}

\paragraph{Results}

We plotted the histogram of error $e_{wc}$ (\autoref{fig:pmi}). We see that the magnitude of $e_{wc}$ of frequency-weighted centering (\freq) is small and
$\PMI(w,c) = \bm{v}_w^\top \bm{u}_c$ holds more accurately \emph{regardless of whether the method is SGNS or GloVe}.
It is worth noting that frequency-weighted centering (\freq) and uniform centering (\unif) have substantially different results, which is non-trivial.

\begin{figure}[tb]
    \centering
\includegraphics[width=77mm]{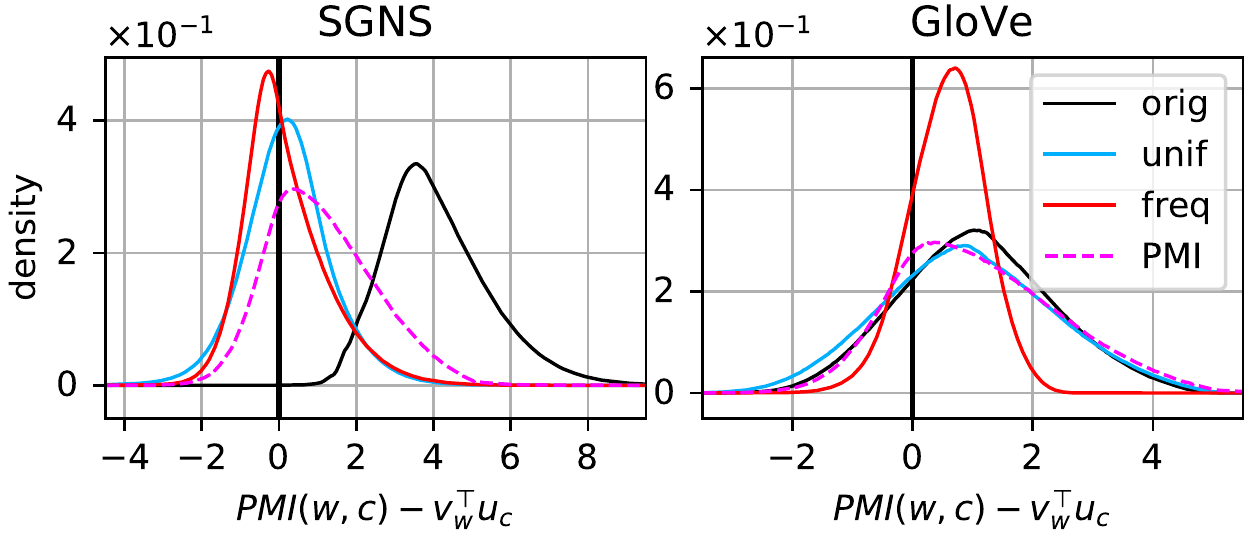}
    \caption{The distribution of the error $e_{wc}$ is plotted for the word pairs $(w,c)$ that co-occur more than once. The distribution of $\PMI(w,c)$ itself is shown as \textcolor[rgb]{0.8,0,0.8}{purple dashed line} for reference of the error order. }
    \label{fig:pmi}
\end{figure}

\subsection{Assessing Accuracy of AND Formula}\label{subsec:exp_and}

From \autoref{th:pmi_th}, \autoref{prop:pmi_order} and \citet{allen}, frequency-weighted centering is expected to result in stronger additive compositionality (\autoref{sec:pmi}). In this section, we experimentally confirm that additive compositionality holds more accurately by frequency-weighted centering. The experiments are for three types of additive compositionality: word-to-sentence, word-to-phrase, and word-to-word.  We also experimentally saw that the same result holds for BERT as well as SGNS and GloVe.

\paragraph{Embeddings} We used 300-dimensional embeddings trained by SGNS and GloVe with Wikipedia\footnote{\url{https://dumps.wikimedia.org/}}. For BERT embeddings, we used the first layer, which corresponds to the target vector of the skip-gram\footnote{\url{https://huggingface.co/bert-base-uncased}}.
In the word-to-sentence experiment, the results for the final layer are also included (\autoref{subsec:add_exp_res_sts_bert})\footnote{In the word-to-word experiment, only the results for the first layer are included because there is no point in contextualizing the word embedding.}.
We compared four types of embeddings: \orig, \unif, \freq~and All-but-the-Top \cite{abtt} (\abtt, for SGNS and GloVe), a post-processing method for embeddings that incorporates uniform centering.

\subsubsection{Word-to-sentence compositionality}  \label{subsubsec:w2s}

\begin{table*}[tb]
    \centering
    \begin{tabular}{c|cccc|cccc|ccc}
        \bhline{0.8pt}
         & \multicolumn{4}{c|}{SGNS} & \multicolumn{4}{c|}{GloVe} & \multicolumn{3}{c}{BERT} \\
        STS$x$ & \orig & \unif & \abtt & \freq & \orig & \unif & \abtt & \freq & \orig & \unif & \freq \\
        \hline \hline
        12 & 0.53 & 0.53 & 0.52 & \bf{0.55} & 0.32 & 0.32 & 0.35 & \bf{0.37} & 0.52 & 0.49 & \bf{0.53} \\
        13 & 0.59 & 0.59 & 0.57 & \bf{0.64} & 0.37 & 0.37 & 0.43 & \bf{0.48} & 0.49 & 0.47 & \bf{0.52} \\
        14 & 0.60 & 0.60 & 0.59 & \bf{0.68} & 0.38 & 0.38 & 0.44 & \bf{0.52} & 0.57 & 0.54 & \bf{0.62} \\
        15 & 0.62 & 0.62 & 0.60 & \bf{0.70} & 0.44 & 0.44 & 0.48 & \bf{0.55} & 0.60 & 0.57 & \bf{0.68} \\
        16 & 0.55 & 0.55 & 0.53 & \bf{0.65} & 0.33 & 0.33 & 0.38 & \bf{0.51} & 0.60 & 0.57 & \bf{0.69} \\
        \bhline{0.8pt}
    \end{tabular}
    \caption{Results of semantic textual similarity tasks. The values are Pearson's correlation coefficients between manually evaluated similarities of sentence pairs and cosine similarities of sentence vector pairs. }
    \label{tbl:sts}
\end{table*}

We evaluated sentence vectors by simply adding word vectors using semantic textual similarity task \cite{sts}. If additive compositionality holds more accurately, it is expected that sentence vectors are more accurate and scores increase.

\paragraph{Results} The results are shown in \autoref{tbl:sts}. 
These values are Pearson's correlation coefficients between the similarity of the two sentences (manually  evaluated) and the cosine similarity between the two sentence vectors.
As we can see, our proposed method \freq~consistently performs the best. All-but-the-Top is also a post-processing method for adjusting the embeddings \cite{arora-latent}, but ours is better in terms of additive compositionality.

\subsubsection{Word-to-phrase compositinality}

\begin{figure}[tb]
    \centering
\includegraphics[width=77mm]{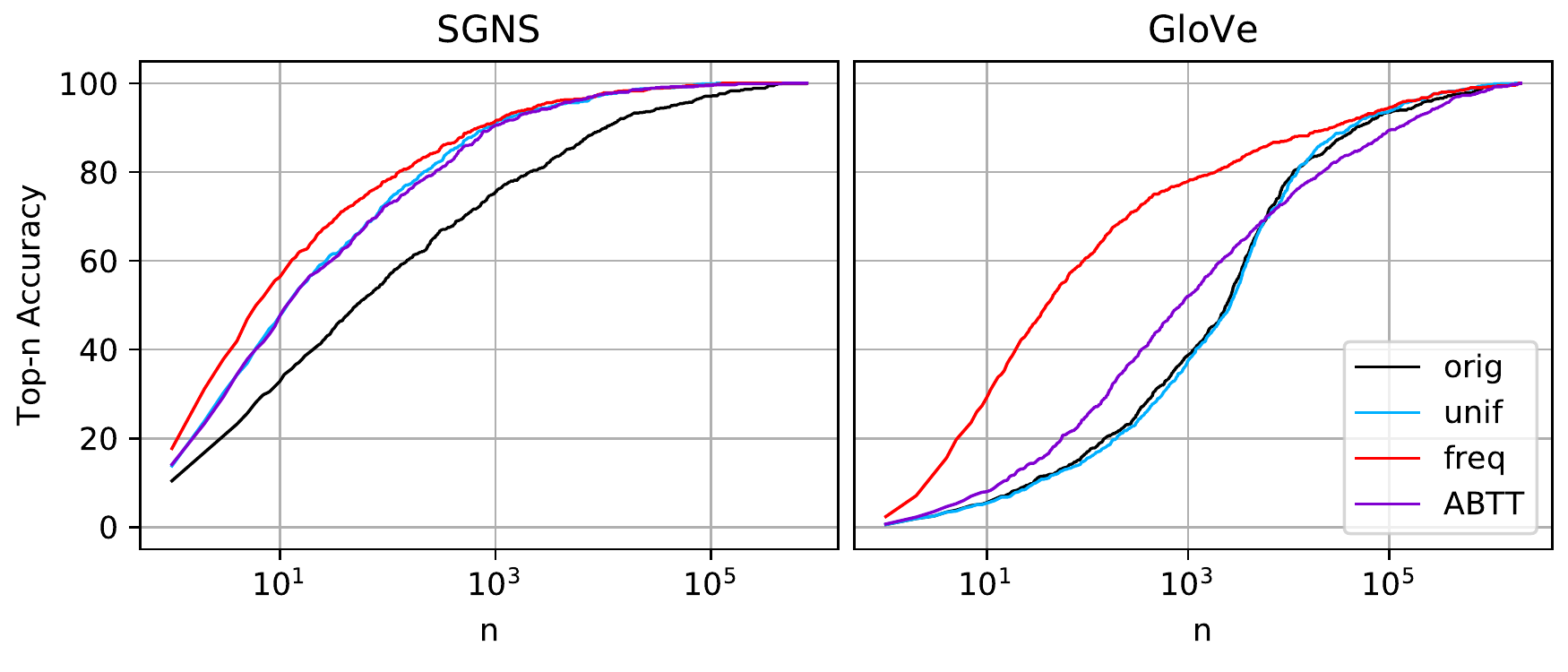}
    \caption{Top-$n$ accuracy of rank (word-to-phrase). 
    Upper left is better.}
    
    \label{fig:and}
\end{figure}

We evaluated how strongly word-to-phrase additive compositionality holds by learning phrase vectors.

\paragraph{Preprocessing of Corpus} We trained phrase vectors by treating multiple words as single word,
\ie, \word{card game} $\rightarrow$ \word{card\_game}. Only phrases with high compositionality included in \citep{multiword,ramisch,reddy} were used\footnote{These datasets include human ratings of the compositionality of phrases. Since words with weak compositionality are not suitable for the additive compositionality experiment, only phrases with a rating of 3 or 4 were used in \citet{multiword} and only phrases with a rating of 3.0 or higher were used in \citet{reddy} or \citet{ramisch}. }.

\paragraph{Evaluation} We calculated the cosine similarities between $\bm{v}_\textit{word1} + \bm{v}_\textit{word2}$ and all the $\bm{v}_w, w \in V$, and how many words had a cosine similarity greater than or equal to the cosine similarity between $\bm{v}_\textit{word1\_word2}$ and $\bm{v}_\textit{word1} + \bm{v}_\textit{word2}$; 
this number is simply denoted as \emph{rank}\footnote{We should not simply use the similarity between $\bm{v}_\textit{word1} + \bm{v}_\textit{word2}$ and $\bm{v}_\textit{word1\_word2}$ as the accuracy of additive compositionality. To be precise, if the similarity with $\bm{v}_\textit{word1\_word2}$ is high and the similarity with other embeddings is low, we can say that additive compositionality is accurate. This paper uses rank as a metric that satisfies this requirement.}.

\paragraph{Results} 
The top-$n$ accuracy of rank is shown in \autoref{fig:and}.
The overall results show that centering, especially with frequency weights, improves the accuracy for additive compositionality. For SGNS, the top-10 accuracy improves by 1.7 times, and for GloVe, the top-100 accuracy improves by 3.5 times. Moreover, the results for GloVe are significantly different between uniform and frequency-weighted centering, which is consistent with the results in \autoref{subsec:exp_pmi}.

\subsubsection{Word-to-word compositinality}

\begin{table}[tb]
    \centering
    \begin{tabular}{c|cccc}
        \hline
         & \orig & \unif & \abtt & \freq \\
         \hline
        SGNS & 0.028 & 0.072 & \bf{0.074} & 0.071 \\
        GloVe & 0.067 & 0.065 & 0.057 & \bf{0.078} \\
        BERT & 0.036 & 0.044 & --- & \bf{0.062} \\
        \hline
    \end{tabular}
    \caption{MRR of rank (word-to-word).}
    \label{tbl:mrr_w2w}
\end{table}

We evaluated the additive compositionality of a word from words such as \word{royal}+\word{woman}=\word{queen}.

\paragraph{Dataset} BATS \cite{bats}, the dataset for the analogy task, specifies the relationship between the two words: the file \texttt{country-capital} contains word pairs such as \word{bankok}:\word{thailand} and \word{beijing}:\word{china}. Using BATS, we created a dataset of word triples $(x, y, z)$ that semantically satisfy $x+y=z$. For example, we assign \word{thailand} to $x$, \word{capital} to $y$, and \word{bankok} to $z$, where $y$ is derived from the dataset name \texttt{country-capital}. There are a total of 9 datasets other than \texttt{country-capital}, (\eg\; \texttt{animal-sound}, \texttt{things-color}), and $y$ is determined in the same way for each.

\paragraph{Evaluation} We used ranks of $\bm{v}_x + \bm{v}_y$ and $\bm{v}_z$ for evaluation. Mean Reciprocal Rank (MRR) was used as the representative value.

\paragraph{Results} \autoref{tbl:mrr_w2w} shows the results. One can see that the proposed method \freq consistently contributes to the performance improvement of additive compositionality. We can also see that, for GloVe and BERT, \freq is superior to the other methods. \freq loses to \unif in SGNS, but this is related to the lack of significant difference in the structure of embeddings between \unif and \freq, as can be seen in \autoref{fig:pmi}.

\subsection{Assessing Accuracy of OR Formula}\label{subsec:exp_or}

In this section, we confirm that the OR formula (\ref{eq:or_th_v}) is valid.

\paragraph{Embeddings}
We used 300-dimensional embeddings trained by SGNS and GloVe with a Wikipedia-based corpus.

\subsubsection{Experiments with artificial OR words}\label{subsec:exp_or_art}

We tested the validity of our theory by artificially creating OR words that exactly satisfy the modeling of OR.

\paragraph{Preprocessing of Corpus}
We generated 500 artificial OR words and learned their embeddings as follows.
We constructed artificial OR words from two randomly selected words (\eg~\textit{apple}, \textit{banana} $\rightarrow$ \textit{apple\_OR\_banana}), and created a new corpus in which all the selected words are replaced by the artificial OR words.
Then we concatenated the original corpus with the new corpus and used it to train word embeddings.

\paragraph{Evaluation}
First, we evaluated the OR formula by cosine similarity and rank of \textit{word1\_OR\_word2}, as in \autoref{subsec:exp_and}.

\paragraph{Results} The average cosine similarity was \textbf{0.936} for SGNS and \textbf{0.907} for GloVe. The average rank was \textbf{1.012} for SGNS and \textbf{1.000} for GloVe. Even though the OR formula is an approximation, the precision of the OR formula is high enough that it is almost always able to predict the correct answer (among 2M words). Since the OR formula is translation-invariant, the result is the same for \orig, \unif, and \freq.

\subsubsection{Experiments with actual OR words}\label{subsec:exp_or_act}

Next, we used actual polysemous words and hypernym-hyponym to evaluate the OR formula. These words do not necessarily satisfy OR modeling exactly, unlike artificial ones. 
We created a dataset of tuples $(w, w_1, ..., w_s)$ satisfying $w = w_1 \lor \cdots \lor w_s$ in the manner described in the \autoref{subsec:apx_5354}. For example, $\word{electronics} = \word{computer} \lor \word{phone}, \quad \word{interaction} = \word{contact} \lor \word{give-and-take} \lor \word{interplay} \lor \word{reciprocation} \lor \word{interchange}$.

\autoref{tbl:hyper-hypo-or} shows the average cosine similarity between $\bm{v}_w$ and the vector calculated by the OR formula from $\bm{v}_{w_1}, \ldots, \bm{v}_{w_s}$. The value in parentheses is the average of the cosine similarity when $w$ is randomly chosen from the dataset. The cosine similarities are significantly high,
indicating that the OR formula works well. Taking $\word{computer} \lor \word{phone} = \word{electronics}$ as an example, the top ten words with high cosine similarity to $\frac{p(\word{computer})}{p(\word{computer}) + p(\word{phone})} \bm{v}_{\word{computer}} + \frac{p(\word{phone})}{p(\word{computer}) + p(\word{phone})} \bm{v}_{\word{phone}}$ are \word{computers}, \word{software}, \word{technology}, \word{internet}, \word{computing}, \word{devices}, \word{electronics}, \word{device}, \word{information}, \word{user}. The words are abstracted from \word{computer} and \word{phone}, and it can be seen that the OR formula is appropriate.

\begin{table}[tb]
    \centering
    \begin{tabular}{c|cc}
        \hline
         & SGNS & GloVe \\
         \hline
        WordNet(noun) & 0.394 (0.199) & 0.315 (0.148) \\
        WordNet(verb) & 0.356 (0.234) & 0.322 (0.214) \\
        Dasgupta et al. & 0.444 (0.273) & 0.439 (0.288) \\
        \hline
    \end{tabular}
    \caption{Cosine-similarity based accuracy of OR formula with the actual corpus (The values in parentheses are the result of randomly selecting word $w$).}
    \label{tbl:hyper-hypo-or}
\end{table}

\subsection{Observation of NOT Formula} \label{subsec:exp_cond}

The visualization of the embeddings of the numbers from -9 to 9 is shown in \autoref{fig:not}\footnote{The origin of the conditional embedding is calculated with uniform weights instead of the actual frequencies. }. 
Confining attention to $A=\{1,\ldots,9\}$,
1 and 9 are located in the negative direction of each other across the origin (red triangle) in the conditional embedding of  $A$; this confirms the NOT formula (\ref{eq:not-formula-theorem}). 
On the other hand, if we expand set $A$ to include negative numbers, 1 and 9 are located in a similar direction from the origin (black $\times$). In this case, the positive and negative numbers are on the opposite sides of the origin, which supports the NOT formula again. As you can see, it is important to determine the category in which antonyms are considered, and the NOT formula is able to formulate this fact well.

\begin{figure}[tb]
    \begin{center}
        \includegraphics[width=55mm]{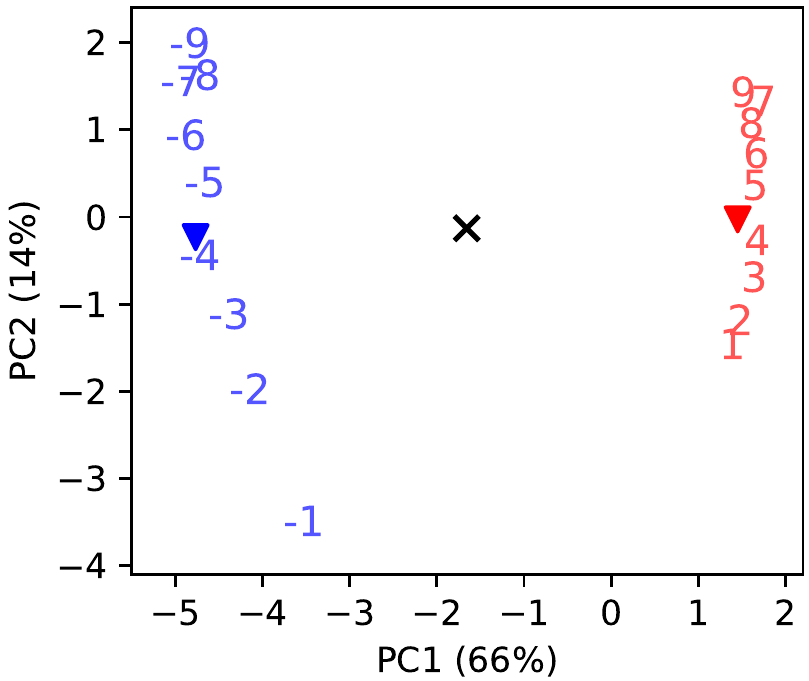}
    \end{center}
    \caption{The embeddings of numbers are visualized using PCA. The $\times$ is the origin of the conditional embedding $\bm{v}_{w|A}$ when $A = \{-9, \ldots, 9\}\ \setminus\{0\}$, and \textcolor{red}{red triangle} and \textcolor{blue}{blue triangle} are the origins of $A$ for positive numbers and negative numbers, respectively.}
    \label{fig:not}
\end{figure}

\section{Connection to Previous Work}

The summaries of the previous research on additive compositionality and their relationships to this study are given below.
\begin{itemize}
    \setlength{\leftskip}{-0.5em}
    \item \citet{arora-latent,arora-polysemy} explained the operations of analogy and OR by considering a latent variable model.
    On the other hand, there is a slight gap between the embedding properties suggested by their theory and the properties of the word embeddings used in practice. For example, their theory shows that high-frequency words have a large norm, but in actual word embeddings, the norm of medium and low-frequency words is large \cite{adriaan}, and this is one of the reasons why additive construction works well \cite{yokoi}. Our theory describes a more realistic embedding model.
    \item \citet{gittens} explained the AND operation with the assumption $p(w) = 1/|V|$ in the Skip-Gram model \cite{skipgram}. While their theory succeeds in explaining the essential reason for additive compositionality, note that it makes the assumption that all words have the same frequency, an assumption that does not hold in practice. Word frequencies are known to have a skewed distribution \cite{zipf}, and a feature of our theory is to incorporate this non-uniform distribution into the theory (\autoref{sec:logical}). 
    \item \citet{allen} explained additive compositionality (AND) and analogy operation based on the assumption (\ref{eq:pmi}). 
    Our theory is positioned as contributing to the elaboration of their theory by resolving the problems of the arbitrariness of bias terms in GloVe and the $\log k$ shift in SGNS, which they had raised as an issue in their theory.
    \item \citet{ethayarajh} proved a necessary and sufficient condition on co-occurrence frequency for analogy operation to hold in SGNS. 
    They prove what relationship of co-occurrence statistics exists between words $w_1$, $w_2$, and $w$ when additive compositionality $\bm{v}_{w_1} + \bm{v}_{w_2} = \bm{v}_w$ holds, as a corollary of the theory of analogy. The difference between their theory and ours is that they focus on the co-occurrence statistics between $w_1$, $w_2$, and $w$, whereas we focus on the relationship of co-occurrence statistics between an arbitrary context $c$ and each $w_1$, $w_2$, $w$.
\end{itemize}

\section{Conclusion}

In this paper, we show that when frequency-weighted centering is performed, SGNS and GloVe share a common structure  and additive compositionality becomes more accurate. We also show how to compute OR and NOT operations by word embeddings in addition to the ordinal additive compositionality (AND). All these ideas are connected to each other by the key formula of PMI, which is represented as the inner product of word embeddings.

Our theory is limited to simple models such as SGNS and GloVe.
We experimentally confirmed the effectiveness of our method on BERT in addition to SGNS and GloVe.
In future work, we aim to interpret BERT theoretically and explain the results of these experiments on BERT and clarify the generality of the experimental results.

\nocite{*} 



\newpage

\appendix

\section*{Appendices}

\section{Proof of \autoref{th:pmi_th}} \label{sec:pmi_proof}

\begin{proof}
    For SGNS, let $\zeta_w = 0, \xi_c = \log \frac{q(c)}{p(c)}, \gamma = \log k$; for GloVe, let $\zeta_w = a_w - \log p(w), \xi_c = b_c - \log p(c), \gamma = -\log Z$. Then, from (\ref{eq:sgns_fac}), (\ref{eq:glove_fac}), we get
    \begin{align}
        \PMI(w,c) = \bm{v}_w^\top \bm{u}_c + \zeta_w + \xi_c + \gamma . \label{eq:pmi_th_st}
    \end{align}
    Multiplying both sides of (\ref{eq:pmi_th_st}) by $p(w)$ and summing with respect to $w \in V$, we get
    \begin{align}
        -\epsilon_c = \bm{\bar{v}}^\top \bm{u}_c + \bar\zeta + \xi_c + \gamma , \label{eq:pmi_th_sumj}
    \end{align}
    where $\bar\zeta = \sum_{w \in V} p(w) \zeta_w$. 
    From (\ref{eq:pmi_th_st}) and (\ref{eq:pmi_th_sumj}):
    \begin{align}
        \PMI(w,c) = \bm{\tilde{v}}_w^\top \bm{u}_c + (\zeta_w - \bar\zeta) - \epsilon_c .  \label{eq:pmi_th_cj}
    \end{align}
    Multiplying both sides of (\ref{eq:pmi_th_cj}) by $p(c)$ and summing with respect to $c \in V$, we get
    \begin{align}
        -\epsilon_w = \bm{\tilde{v}}_w^\top \bm{\bar{u}} + (\zeta_w - \bar\zeta) - \bar{\epsilon} \label{eq:pmi_th_sumi}
    \end{align}
    From (\ref{eq:pmi_th_cj}) and (\ref{eq:pmi_th_sumi}), we have
    \begin{align}
        \PMI(w,c) = \bm{\tilde{v}}_w^\top \bm{\tilde{u}}_c + \bar{\epsilon} - \epsilon_w - \epsilon_c \nonumber
    \end{align}
\end{proof}

\section{Proof of \autoref{prop:pmi_order}} \label{sec:pmi_order_proof}

\begin{proof}
    There exists $c_1 > 0$ such that for all $(w,c) \in V^2$, 
    \begin{multline}
        \left| -1 + \frac{p(w,c)}{p(w)p(c)} \right| \\
        = | -1 + \exp(\PMI(w,c)) | < c_1 \Delta .
    \end{multline}
    \begin{align}
        \epsilon_w &= -\sum_{c \in V} p(c) \log \frac{p(w,c)}{p(w)p(c)} \nonumber \\
        &= -\sum_{c \in V} p(c) \biggl[ \left(-1 + \frac{p(w,c)}{p(w)p(c)}\right) \nonumber \\
        & \qquad + O\left(\left|-1 + \frac{p(w,c)}{p(w)p(c)}\right|^2\right) \biggr] \nonumber \\
        &= \sum_{c \in V} p(c) - \sum_{c \in V} p(c|w) \nonumber \\
        & \qquad- \sum_{c \in V} p(c) O\left(\left|-1 + \frac{p(w,c)}{p(w)p(c)}\right|^2\right) \nonumber \\
        &= -\sum_{c \in V} p(c) O\left(\left|-1 + \frac{p(w,c)}{p(w)p(c)}\right|^2\right) .
    \end{align}
    Therefore, there exists $c_2 > 0$ such that for all $w \in V$,
    \begin{align}
        |\epsilon_w| < \sum_{c \in V} p(c) c_2 c_1^2 \Delta^2 = c_1^2 c_2 \Delta^2 .
    \end{align}
    $|\bar\epsilon| < c_1 c_2 \Delta^2$ also readily follows.
\end{proof}

\section{Proof of \autoref{th:or_th}} \label{sec:or_proof}

\begin{proof}
    Calculating both sides of (\ref{eq:or_modeling}), we get
    \begin{align}
        & p(w|c) = p(w) \exp(\PMI(w, c)) \nonumber \\
        & \qquad \approx p(w) \left(1 + \PMI(w, c)\right) \nonumber \\
        & \qquad = p(w) (1 + \bm{v}_w^\top \bm{u}_{c}), \label{eq:or_th_left} \\
        & \sum_{i=1}^s p(w_i|c) = \sum_{i=1}^s p(w_i) \exp(\PMI(w_i, c)) \nonumber \\
        & \qquad \approx \sum_{i=1}^s p(w_i) (1 + \bm{v}_{w_i}^\top \bm{u}_{c}) \nonumber \\
        & \qquad = p(w) \left[1 + \left(\sum_{i=1}^s \frac{p(w_i)}{p(w)} \bm{v}_{w_i} \right)^\top \bm{u}_{c}\right] . \label{eq:or_th_right}
    \end{align}
    (\ref{eq:or_th_v}) follows the fact that for any $c \in V$, (\ref{eq:or_th_left}) $\approx$ (\ref{eq:or_th_right}) . 
\end{proof}

\section{Proof of \autoref{th:cond_emb}} \label{sec:cond_emb_proof}

\begin{proof}
Calculating the left-hand side of (\ref{eq:cond_model}) using assumption (\ref{eq:pmi}) and OR formula (\ref{eq:or_th_v}), we get
\begin{multline}
    p(W = w \mid W \in A, c) \\
    = \frac{p(W = w, W \in A \mid c)}{p(W \in A \mid c)}
    \approx\frac{p(w) \exp(\bm{v}_{w}^\top \bm{u}_{c})}{p(A) \exp(\bm{v}_A^\top \bm{u}_{c})} \\
    = p(W = w \mid W \in A) \exp((\bm{v}_{w}- \bm{v}_A)^\top \bm{u}_{c}), \label{eq:cond_cooccur_prob}
\end{multline}
By comparing (\ref{eq:cond_model}) and (\ref{eq:cond_cooccur_prob}), we get (\ref{eq:cond_emb_calc}).
\end{proof}

\section{Proof of \autoref{th:not_th}} \label{sec:not_proof}

\begin{proof}
By using (\ref{eq:or_th_v}), the right-hand side of (\ref{eq:not_modeling}) is rearranged as
\begin{multline}
    p(W \in A \setminus \{w\} \mid W \in A, c) \\
    = \frac{p(W\in A \setminus \{w\} \mid c)}{p(W \in A \mid c)} \\
    \approx \frac{p(A \setminus \{w\})}{p(A)} \exp((\bm{v}_{A \setminus \{w\}} - \bm{v}_A)^\top \bm{u}_{c}) .
\end{multline}
Thus $\bm{v}_{\lnot w | A} \approx \bm{v}_{A \setminus \{w\}} - \bm{v}_A$, and further calculation yields
\begin{multline}
    \bm{v}_{\lnot w | A} 
    \approx \frac{p(A)}{p(A) - p(w)} \left(\bm{v}_A - \frac{p(w)}{p(A)}\bm{v}_{w} \right) - \bm{v}_A \\
    = -\frac{p(W = w \mid W \in A)}{1-p(W = w \mid W \in A)} \bm{v}_{w | A} . \label{eq:not_th}
\end{multline}
\end{proof}

\section{Details of Experiments} \label{sec:detail_exp}

The default parameters of the implementation\footnote{\url{https://github.com/tmikolov/word2vec},  \url{https://github.com/stanfordnlp/GloVe}} were used for all but the most notable cases.

\subsection{Details of \autoref{subsec:exp_pmi}}

\paragraph{Corpus} text8 corpus\footnote{\url{http://mattmahoney.net/dc/textdata.html}}, from which low-frequency words ($<100$) were removed.
\paragraph{Hyperparameters for learning word embeddings} We run 100 iterations for 300-dimensional vectors. The size of the context window is 5 words (symmetric context). For SGNS, the number of negative samples $k$ is 15 and subsampling of high-frequency words was disabled. For GloVe, the parameter for the weights of the least-squares method $x_{\max}$ is 100.
\paragraph{Others} For \freq and \unif, $\bm{u}_c$ is also centered.

\subsection{Details of \autoref{subsec:exp_and} and \autoref{subsec:exp_or}}\label{subsec:apx_5354}

\paragraph{Corpus} Wikipedia\footnote{\url{https://dumps.wikimedia.org/}}(2.1G tokens)
\paragraph{Hyperparameters for learning word embeddings} The dimension of the word embeddings is 300 and the size of the context window is 5 words. For SGNS, the number of negative samples $k$ is 15. For GloVe, the parameter for the weights of the least-squares method $x_{\max}$ is 100.
\paragraph{Artificial OR words} In \autoref{subsec:exp_or_art}, the words used to construct artificial polysemous words were those with more than 100 occurrences. In the calculation of rank, \textit{word1} and \textit{word2} were excluded from the search.
\paragraph{Actual OR words} 
Each dataset shown in \autoref{tbl:hyper-hypo-or} was created in the following way.
\begin{description}
    \item[Dasgupta et al.] We used the dataset of OR words manually constructed by \citet{dasgupta}. It mainly contains homographs that satisfy $w \approx w_1 \lor w_2$. The sample size was 22.
    \item[WordNet(noun)] By extracting hypernym-hyponym relations from WordNet, we created a set of tuples $(w, w_1, \cdots, w_s)$ satisfying $w \approx w_1 \lor \cdots \lor w_s$. We used nouns with a frequency of 100 or more. The sample size was 985.
    \item[WordNet(verb)] It differs from WordNet(noun) only in that contained words are verbs. The sample size was 530.
\end{description}

\subsection{Details of \autoref{subsec:exp_cond}}

\paragraph{Word embeddings} GloVe pre-trained with Common Crawl (840G tokens)\footnote{\url{https://nlp.stanford.edu/projects/glove/}}
\paragraph{Others} 0 is not used because the sign cannot be defined.

\section{Additional Experimental Results}

\subsection{\autoref{subsubsec:w2s}} \label{subsec:add_exp_res_sts_bert}

The results of semantic textual similarity using the final layer of BERT are shown in \autoref{tbl:sts_final_bert}. It can be seen that \freq is almost consistently the best.

\begin{table}[tb]
    \centering
    \begin{tabular}{c|ccc}
        \hline
         & \orig & \unif & \freq \\
        \hline
        STS12 & \bf{0.350} & 0.335 & 0.334 \\
        STS13 & 0.254 & 0.277 & \bf{0.280} \\
        STS14 & 0.377 & 0.391 & \bf{0.402} \\
        STS15 & 0.457 & 0.486 & \bf{0.493} \\
        STS16 & 0.446 & 0.455 & \bf{0.472} \\
        \hline
    \end{tabular}
    \caption{The results of semantic textual similarity using the final layer of BERT. The values are Pearson's correlation coefficients between manually evaluated similarities of sentence pairs and cosine similarities of sentence vector pairs. }
    \label{tbl:sts_final_bert}
\end{table}


\begin{thebibliography}{27}
\expandafter\ifx\csname natexlab\endcsname\relax\def\natexlab#1{#1}\fi

\bibitem[{Agirre et~al.(2012)Agirre, Cer, Diab, and Gonzalez-Agirre}]{sts}
Eneko Agirre, Daniel Cer, Mona Diab, and Aitor Gonzalez-Agirre. 2012.
\newblock {S}em{E}val-2012 task 6: A pilot on semantic textual similarity.
\newblock In \emph{*{SEM} 2012: The First Joint Conference on Lexical and
  Computational Semantics {--} Volume 1: Proceedings of the main conference and
  the shared task, and Volume 2: Proceedings of the Sixth International
  Workshop on Semantic Evaluation ({S}em{E}val 2012)}, pages 385--393.

\bibitem[{Allen and Hospedales(2019)}]{allen}
Carl Allen and Timothy Hospedales. 2019.
\newblock Analogies explained: Towards understanding word embeddings.
\newblock In \emph{Proceedings of the 36th International Conference on Machine
  Learning}, volume~97, pages 223--231.

\bibitem[{Arora et~al.(2016)Arora, Li, Liang, Ma, and Risteski}]{arora-latent}
Sanjeev Arora, Yuanzhi Li, Yingyu Liang, Tengyu Ma, and Andrej Risteski. 2016.
\newblock A latent variable model approach to {PMI}-based word embeddings.
\newblock \emph{Transactions of the Association for Computational Linguistics},
  4:385--399.

\bibitem[{Arora et~al.(2018)Arora, Li, Liang, Ma, and
  Risteski}]{arora-polysemy}
Sanjeev Arora, Yuanzhi Li, Yingyu Liang, Tengyu Ma, and Andrej Risteski. 2018.
\newblock {Linear Algebraic Structure of Word Senses, with Applications to
  Polysemy}.
\newblock \emph{Transactions of the Association for Computational Linguistics},
  6:483--495.

\bibitem[{Cruse(1986)}]{cruse}
David~A. Cruse. 1986.
\newblock \emph{Lexical Semantics}.
\newblock Cambridge University Press, Cambridge, UK.

\bibitem[{Dasgupta et~al.(2022)Dasgupta, Boratko, Mishra, Atmakuri, Patel, Li,
  and McCallum}]{dasgupta}
Shib Dasgupta, Michael Boratko, Siddhartha Mishra, Shriya Atmakuri, Dhruvesh
  Patel, Xiang Li, and Andrew McCallum. 2022.
\newblock {Word2Box: Capturing Set-Theoretic Semantics of Words using Box
  Embeddings}.
\newblock In \emph{Proceedings of the 60th Annual Meeting of the Association
  for Computational Linguistics (Volume 1: Long Papers)}, pages 2263--2276.

\bibitem[{Devlin et~al.(2019)Devlin, Chang, Lee, and Toutanova}]{bert}
Jacob Devlin, Ming-Wei Chang, Kenton Lee, and Kristina Toutanova. 2019.
\newblock {BERT}: {Pre-training of Deep Bidirectional Transformers for Language
  Understanding}.
\newblock In \emph{Proceedings of the 2019 Conference of the North {A}merican
  Chapter of the Association for Computational Linguistics: Human Language
  Technologies}, pages 4171--4186.

\bibitem[{Ethayarajh et~al.(2019)Ethayarajh, Duvenaud, and Hirst}]{ethayarajh}
Kawin Ethayarajh, David Duvenaud, and Graeme Hirst. 2019.
\newblock Towards understanding linear word analogies.
\newblock In \emph{Proceedings of the 57th Annual Meeting of the Association
  for Computational Linguistics}, pages 3253--3262.

\bibitem[{Farahmand et~al.(2015)Farahmand, Smith, and Nivre}]{multiword}
Meghdad Farahmand, Aaron Smith, and Joakim Nivre. 2015.
\newblock {A Multiword Expression Data Set: Annotating Non-Compositionality and
  Conventionalization for English Noun Compounds}.
\newblock In \emph{Proceedings of the 11th Workshop on Multiword Expressions},
  pages 29--33.

\bibitem[{Firth(1957)}]{firth}
J.~R. Firth. 1957.
\newblock A synopsis of linguistic theory 1930-55.
\newblock \emph{Studies in Linguistic Analysis (special volume of the
  Philological Society)}, 1952-59:1--32.

\bibitem[{Gittens et~al.(2017)Gittens, Achlioptas, and Mahoney}]{gittens}
Alex Gittens, Dimitris Achlioptas, and Michael~W. Mahoney. 2017.
\newblock Skip-gram - {Z}ipf + {U}niform = {V}ector {A}dditivity.
\newblock In \emph{Proceedings of the 55th Annual Meeting of the Association
  for Computational Linguistics}, pages 69--76.

\bibitem[{Gladkova et~al.(2016)Gladkova, Drozd, and Matsuoka}]{bats}
Anna Gladkova, Aleksandr Drozd, and Satoshi Matsuoka. 2016.
\newblock Analogy-based detection of morphological and semantic relations with
  word embeddings: what works and what doesn{'}t.
\newblock In \emph{Proceedings of the {NAACL} Student Research Workshop}, pages
  8--15.

\bibitem[{Grbovic et~al.(2015)Grbovic, Radosavljevic, Djuric, Bhamidipati,
  Savla, Bhagwan, and Sharp}]{prod2vec}
Mihajlo Grbovic, Vladan Radosavljevic, Nemanja Djuric, Narayan Bhamidipati,
  Jaikit Savla, Varun Bhagwan, and Doug Sharp. 2015.
\newblock {E-Commerce in Your Inbox: Product Recommendations at Scale}.
\newblock In \emph{Proceedings of the 21th ACM SIGKDD International Conference
  on Knowledge Discovery and Data Mining}, page 1809–1818.

\bibitem[{Grover and Leskovec(2016)}]{node2vec}
Aditya Grover and Jure Leskovec. 2016.
\newblock {Node2vec: Scalable Feature Learning for Networks}.
\newblock In \emph{Proceedings of the 22nd ACM SIGKDD International Conference
  on Knowledge Discovery and Data Mining}, page 855–864.

\bibitem[{Harris(1954)}]{harris54}
Zellig Harris. 1954.
\newblock Distributional structure.
\newblock \emph{Word}, 10(2-3):146--162.

\bibitem[{Levy and Goldberg(2014)}]{levy}
Omer Levy and Yoav Goldberg. 2014.
\newblock {Neural Word Embedding as Implicit Matrix Factorization}.
\newblock In \emph{Advances in Neural Information Processing Systems}, pages
  2177--2185.

\bibitem[{Mikolov et~al.(2013{\natexlab{a}})Mikolov, Chen, Corrado, and
  Dean}]{skipgram}
Tom{\'{a}}s Mikolov, Kai Chen, Greg Corrado, and Jeffrey Dean.
  2013{\natexlab{a}}.
\newblock Efficient estimation of word representations in vector space.
\newblock In \emph{1st International Conference on Learning Representations,
  {ICLR}, Workshop Track Proceedings}.

\bibitem[{Mikolov et~al.(2013{\natexlab{b}})Mikolov, Sutskever, Chen, Corrado,
  and Dean}]{sgns}
Tom{\'{a}}s Mikolov, Ilya Sutskever, Kai Chen, Greg Corrado, and Jeffrey Dean.
  2013{\natexlab{b}}.
\newblock Distributed representations of words and phrases and their
  compositionality.
\newblock In \emph{Advances in Neural Information Processing Systems}, pages
  3111--3119.

\bibitem[{Mu and Viswanath(2018)}]{abtt}
Jiaqi Mu and Pramod Viswanath. 2018.
\newblock All-but-the-top: Simple and effective postprocessing for word
  representations.
\newblock In \emph{International Conference on Learning Representations}.

\bibitem[{Pennington et~al.(2014)Pennington, Socher, and Manning}]{glove}
Jeffrey Pennington, Richard Socher, and Christopher~D. Manning. 2014.
\newblock Glove: Global vectors for word representation.
\newblock In \emph{Proceedings of the 2014 Conference on Empirical Methods in
  Natural Language Processing}, pages 1532--1543.

\bibitem[{Piantadosi(2014)}]{zipf}
Steven Piantadosi. 2014.
\newblock {Zipf's word frequency law in natural language: A critical review and
  future directions}.
\newblock \emph{Psychonomic bulletin and review}, 21.

\bibitem[{Ramisch et~al.(2016)Ramisch, Cordeiro, Zilio, Idiart, and
  Villavicencio}]{ramisch}
Carlos Ramisch, Silvio Cordeiro, Leonardo Zilio, Marco Idiart, and Aline
  Villavicencio. 2016.
\newblock {How Naked is the Naked Truth? A Multilingual Lexicon of Nominal
  Compound Compositionality}.
\newblock In \emph{Proceedings of the 54th Annual Meeting of the Association
  for Computational Linguistics}, pages 156--161.

\bibitem[{Reddy et~al.(2011)Reddy, McCarthy, and Manandhar}]{reddy}
Siva Reddy, Diana McCarthy, and Suresh Manandhar. 2011.
\newblock {An Empirical Study on Compositionality in Compound Nouns}.
\newblock In \emph{Proceedings of 5th International Joint Conference on Natural
  Language Processing}, pages 210--218.

\bibitem[{Roy et~al.(2018)Roy, Ganguly, Bhatia, Bedathur, and Mitra}]{info_ret}
Dwaipayan Roy, Debasis Ganguly, Sumit Bhatia, Srikanta Bedathur, and Mandar
  Mitra. 2018.
\newblock {Using Word Embeddings for Information Retrieval: How Collection and
  Term Normalization Choices Affect Performance}.
\newblock In \emph{Proceedings of the 27th ACM International Conference on
  Information and Knowledge Management}, page 1835–1838.

\bibitem[{Schakel and Wilson(2015)}]{adriaan}
Adriaan M.~J. Schakel and Benjamin~J. Wilson. 2015.
\newblock {Measuring Word Significance using Distributed Representations of
  Words}.
\newblock \emph{CoRR}, abs/1508.02297.

\bibitem[{Willners(2001)}]{willners}
Caroline Willners. 2001.
\newblock \emph{Antonyms in Context : A Corpus-Based Semantic Analysis of
  Swedish Descriptive Adjectives}.
\newblock Ph.D. thesis, General Linguistics.

\bibitem[{Yokoi et~al.(2020)Yokoi, Takahashi, Akama, Suzuki, and Inui}]{yokoi}
Sho Yokoi, Ryo Takahashi, Reina Akama, Jun Suzuki, and Kentaro Inui. 2020.
\newblock Word rotator{'}s distance.
\newblock In \emph{Proceedings of the 2020 Conference on Empirical Methods in
  Natural Language Processing (EMNLP)}, pages 2944--2960.

\end{thebibliography}
\end{document}